\documentclass[twocolumn]{IEEEtran}
\usepackage[switch]{lineno}

\usepackage{ulem}
\usepackage{cite}
\usepackage{amsmath,amssymb,amsfonts, amsthm}
\usepackage{algorithmic}
\usepackage{graphicx}
\usepackage{textcomp}
\usepackage{xcolor}
\usepackage{algorithm}
\usepackage{enumitem}

\newcommand\gnote[1]{{#1}}

\newtheorem{theorem}{Theorem}
\newtheorem{corollary}{Corollary}

\newtheorem{definition}{Definition}
\newtheorem{lemma}[theorem]{Lemma}
\newtheorem{proposition}{Proposition}

\usepackage{hyperref}

\newcommand{\Lemmaref}[1]{Lemma~\ref{#1}}
\newcommand{\propref}[1]{Proposition~\ref{#1}}

\newcommand{\w}{{\mathbf w}}
\newcommand{\x}{{\mathbf x}}
\newcommand{\U}{{\mathbf U}}
\newcommand{\X}{{\mathbf X}}
\newcommand{\e}{{\mathbf e}}
\newcommand{\D}{{\mathbf D}}

\newcommand{\Rt}{{\mathbf R}(\theta)}
\newcommand{\I}{{\mathbf I}}
\newcommand{\V}{{\mathbf V}}
\renewcommand{\O}{{\mathbf O}}
\renewcommand{\P}{{\mathbf P}}
\renewcommand{\v}{{\mathbf v}}
\renewcommand{\u}{{\mathbf u}}

\newcommand{\norm}[1]{\left\Vert #1\right\Vert}

\newcommand{\cO}{\mathcal{O}}
\newcommand{\orth}[2]{\cO^{#1\times #2}}
\newcommand{\mat}[2]{\R^{#1\times #2}}
\newcommand{\R}{\mathbb{R}}

\begin{document}

\title{Fair Principal Component Analysis and Filter Design}

\author{\IEEEauthorblockN{Gad Zalcberg$^{(1)}$ and Ami Wiesel$^{(1)\ (2)}$}

\IEEEauthorblockA{\textit{(1) School of Computer Science and Engineering, The Hebrew University of Jerusalem, Israel\\ (2) Google Research, Israel}}
}

\maketitle

\begin{abstract}
We consider Fair Principal Component Analysis (FPCA) and search for a low dimensional subspace that spans multiple target vectors in a fair manner. FPCA is defined as a non-concave maximization of the worst projected target norm within a given set. The problem arises in filter design in signal processing, and when incorporating fairness into dimensionality reduction schemes. The state of the art approach to FPCA is via semidefinite programming followed by rank reduction methods. Instead, we propose to address FPCA using simple sub-gradient descent. We analyze the landscape of the underlying optimization in the case of orthogonal targets. We prove that the landscape is benign and that all local minima are globally optimal. Interestingly, the SDR approach leads to sub-optimal solutions in this orthogonal case. Finally, we discuss the equivalence between orthogonal FPCA and the design of normalized tight frames. 
\end{abstract}

\begin{IEEEkeywords}
Dimensionality Reduction, SDP, Fairness, Normalized Tight Frame, PCA
\end{IEEEkeywords}

\IEEEpeerreviewmaketitle

\section{Introduction}
\IEEEPARstart{D}{imensionality} reduction is a fundamental problem in signal processing and machine learning. In particular, Principal Component Analysis (PCA) is among the most popular data science tool. It involves a non-concave maximization but has a tight semidefinite relaxation (SDR). Its optimization landscape, saddle points and extreme points are all well understood and it is routinely solved using scalable first order methods \cite{sun2015nonconvex}. PCA maximizes the average performance across a given set of vector targets. In many settings, worst case metrics are preferred in order to ensure fairness and equal performance across all targets. This gives rise to Fair PCA (FPCA) \gnote{\cite{morgenstern2019fair} which will be defined formally in the next section}.  Unfortunately, changing the average PCA objective to a worst case FPCA objective results in an NP-hard problem \cite{morgenstern2019fair} which is poorly understood. There is a growing body of works on convex relaxations via SDR for FPCA \cite{morgenstern2019fair, samadi2018price}, but these methods do not scale well and are inapplicable to many realistic settings. Therefore, the goal of this paper to consider scalable first order solutions to FPCA and shed more light on the landscape of this important optimization problem. 

Due to the significance of PCA it is non-trivial to track the origins of FPCA. In the context of filter design, FPCA with rank one constraints is known as multicast beamforming and there is a huge body of literature on this topic, e.g.,  \cite{sidiropoulos2006transmit,ma2010semidefinite,cheriyadat2003principal}. In the modern context of fairness in machine learning, FPCA was considered in \cite{olfat2019convex,bian2010max,morgenstern2019fair,samadi2018price}. It was shown that SDR with an iterative rounding technique provides near optimal performance when the rank is much larger than the squared root of the number of targets. More generally, by interpreting the worst case operator as an $L_\infty$ norm, FPCA is a special case of $L_{2,p}$ norm optimizations. Classical PCA corresponds to $L_{2,2}$. Robust PCA algorithms as \cite{lerman2015robust} rely on $L_{2,1}$, and FPCA is the other extreme using $L_{2,\infty}$. Most of these works capitalize on the use of SDR that leads to conic optimizations with provable performance guarantees. 
\gnote{Finally, \cite{kamani2019efficient} proposed a different definition to fairness in PCA via multi-objective optimization. They developed a first order method for attaining random solutions on the PCA Pareto frontier. FPCA defined above may be considered as one of these points.}


SDR and nuclear norm relaxations are currently state of the art in a wide range of subspace recovery problems. Unfortunately, SDR is known to scale poorly to high dimensions. Therefore, there is a growing body of works on first order solutions to semidefinite programs. The main trick is to factorize the low rank matrix and show that the landscape of the resulting non-convex objective is benign \cite{burer2003nonlinear,burer2005local,boumal2018deterministic,boumal2016non,cifuentes2019burer,zhu2018global,li2018non}. The SDR of FPCA involves two types of linear matrix inequalities and still poses a challenge. Therefore, we first reformulate the problem and then apply sub-gradient descent on the factorized formulation.

The main contribution of this paper is the observation that the landscape of the factorized FPCA optimization is benign when the targets are orthogonal. This is the case in which dimensionality reduction is most lossy. Yet, we show that it is easy from an optimization perspective. The maximization is non-concave but every (non-connected) local minima is globally optimal. Surprisingly, we show that this case is challenging for SDR. Its objective is tight but it is not trivial to project its solution onto the feasible set. Numerical experiments with synthetic data suggest that these properties also hold in more realistic near-orthogonal settings. Finally, a direct corollary of our analysis is an  equivalence between orthogonal FPCA and the design of finite normalized tight frames \cite{benedetto2003finite}. This characterization may be useful in future works on data-driven normalized tight frame design.
\subsection*{Notations:}
We used bold uppercase letters (e.g. {\bf{P}}) for matrices, bold lowercase letters (e.g. $\v$) for vectors and non-bold letters (e.g. $n$) for scalars. We used pythonic notation for indices of matrices:
$\U_{i:}$ for the $i$'th row, $\U_{:j}$ for the $j$'th column and $\U_{ij}$ for the $i,j$'th entry of matrix. The set of $d\times r$ ($r\le d$) semi-orthogonal matrices (matrices with orthonormal columns) is denoted by $\orth{d}{r}$, the set of positive semidefinite matrices by $\mathbb{S}^d_+$, and the set of $d\times d$ projection matrices of rank at most $r$ by $\mathcal{P}^d_r$ (and $\mathcal{P}^d:=\mathcal{P}^d_d$).
Given a function $f:A\rightarrow\mathbb{R}^m$, $\U\in A$ we define the set of indices $\mathcal{I}_\U:=\arg\max_if_i(\U)$.
Finally we define a projection operator onto the set of projection matrices of rank at most $r$: $\Pi_r:\mathbb{S}^d_+\rightarrow \mathcal{P}^d_r$ as follows: Let $\P=\U{\bf \Sigma}\U^T$ (EVD decomposition), where: $\U=(\u_1,...,\u_d)$ then: $\Pi_r[\P]:=(\u_1,...,\u_r)(\u_1,...,\u_r)^T$.

\section{Problem formulation}
The goal of this paper is to identify a low dimensional subspace that maximizes the smallest norm of a given set of projected targets. More specifically, let $\{\x_i\}_{i=1}^n\subset\mathbb{R}^d$ be the set of targets, we consider the problem: \begin{equation} \label{main_problem}
   {\rm{FPCA:}}\quad \begin{array}{ll}
        \max_{{\bf{P}}\in\mathcal{P}^d} & \min_{i\in[n]}  \x_{i}^{T}{\bf{P}}\x_{i}\\
        {\rm{s.t.}} &  {\rm{rank}}\left({\bf{P}}\right)\le r
    \end{array}
\end{equation}

 Our motivation to FPCA arises in the context of filter design for detection. We are interested in the design of a linear sampling device from ${\mathbb{R}}^d$ to ${\mathbb{R}}^r$ that will allow detection of $n$ known targets denoted by $\{\x_i\}_{i=1}^n$. \gnote{The motivation for using a small rank is that the cost of power, space and/or time resources typically scales with $r$. Detection accuracy in additive white Gaussian noise decreases exponentially with the received signal to noise ratio (SNR), and it is therefore natural to maximize the worst SNR across all the targets. Hopefully, this will lead to a fair solution with equal norms for all the targets.} 
 
  FPCA with $r=1$ is concerned with the design of a single beaamforming filter, and is equivalent to multicast downlink transmit beamforming \cite{sidiropoulos2006transmit,ma2010semidefinite}
\begin{equation}
    \begin{array}{ll}
        \max_{\u\in\mathbb{R}^d} & \min_{i\in[n]} (\x_{i}^{T}\u)^2\\
        {\rm{s.t.}} &  \norm{\u}\le 1
    \end{array}
\end{equation}
Practical systems typically satisfy $r<n\ll d$, e.g., the design of a few expensive sensors that downsample a high resolution digital signal (or even an infinite dimension analog signal). Without loss of optimality, we assume a first stage of dimensionality reduction via PCA that results in effective dimensions such that $n=d$.


As detailed in the introduction, FPCA was also recently introduced  in the context of fair machine learning. There, it is more natural to assume a block structure. The targets are divided into $n$ blocks, denoted by $d\times b_i$ matrices $\X_i$, and fairness needs to be respected with respect to properties as gender or race \cite{olfat2019convex,morgenstern2019fair,samadi2018price}:
\begin{equation} \label{matrix_case}
   {\rm{FPCA}}^{\rm{blocks}}\quad \begin{array}{ll}
        \max_{{\bf{P}}\in\mathcal{P}^d} & \min_{i\in[n]}  {\rm{Tr}}\left(\X_{i}^{T}{\bf{P}}\X_{i}\right)\\
        {\rm{s.t.}} &  {\rm{rank}}\left({\bf{P}}\right)\le r
    \end{array}
\end{equation}
Throughout this paper, we will consider the simpler non-block FPCA formulation corresponding to filter design. Preliminary experiments suggest that most of the results also hold in the block case. 



FPCA is known to be NP-hard \cite{sidiropoulos2006transmit, morgenstern2019fair}. 
The state of the art approach to FPCA is SDR. Specifically, we relax the rank constraint by its convex hull, the nuclear norm, and the projection constraint by linear matrix inequalities \cite{ma2010semidefinite, morgenstern2019fair}. This yields the SDP:
\begin{equation} \label{SDR}
   {\rm{SDR:}}\quad \begin{array}{ll}
        \max_{{\bf{P}}\in\mathbb{S}^d_+} & \min_{i\in[n]} \x_{i}^{T}{\bf{P}}\x_{i}\\
        {\rm{s.t.}} &   {\rm{Trace}}\left({\bf{P}}\right)\le r\\
             &   {\bf{0}}\preceq {\bf{P}}\preceq {\bf{I}}
    \end{array}
\end{equation}
\gnote{The computational complexity of solving an SDR using an Interior Point method is intractable for most applications, but \cite{morgenstern2019fair} propose a practical and efficient multiplicative weight update.} Unfortunately, the optimal solution to SDR might not be a feasible projection, and $\P_{\rm{SDR}}$ may have any rank. Due to the relaxation, SDR always results in an upper bound on FPCA. To obtain a feasible approximation, it is customary to define
\begin{equation}
    \P_{\rm{PrSDR}}=\Pi_r[\P_{\rm{SDR}}]
\end{equation}
PrSDR is a feasible projection matrix of rank $r$, and is therefore a lower bound on FPCA. Better approximations may be obtained via randomized procedures \cite{ma2010semidefinite}. Recently, an iterative rounding technique was proven to provide 
a $\left(1-\frac{O(\sqrt{n})}{r}\right)$ approximation \cite{morgenstern2019fair}. This result is near optimal in the block case where it is reasonable to assume $r\gg \sqrt{n}$. It is less applicable to filter design where $n$ is large and smaller ranks are required.  

The goal of this paper is to provide a scalable, yet accurate solution to FPCA, \gnote{without the need for additional rank reduction schemes.} Motivated by the growing success of simple gradient based methods in complex optimization problems, e.g., deep learning, we consider the application of sub-gradient descent to FPCA and analyze its optimization landscape.

\section{Algorithm}

In this section, we propose an alternative and more scalable approach for solving SDR. The two optimization challenges in FPCA are the projection and rank constraints. We confront the first challenge by reformulating the problem using a quadratic objective, and the second by decomposing the projection matrix using its low rank factors. Together, we define factorized FPCA:
\begin{equation} \label{factorized_restore_problem}
{\rm{F-FPCA:}}\quad\begin{array}{ll}
     \min_{\U\in\mathbb{R}^{d\times r}}  & \max_{{i\in[n]}}\quad f_i(\U)
\end{array}
\end{equation}
where
\begin{equation} \label{fi}
f_i(\U)=\norm{\x_{i} - \U\U^T\x_{i}}^2 - \norm{\x_{i}}^2
\end{equation}
The formal equivalence between \eqref{main_problem} and \eqref{factorized_restore_problem} is stated below.

\begin{proposition} \label{FPCA equivalent to F-FPCA}
Let $\U$ be a globally optimal solution to F-FPCA in \eqref{factorized_restore_problem}. Then, 
$\P=\Pi[{\U}{\U}^T]$ is a globally optimal solution to FPCA in \eqref{main_problem}.
\end{proposition}
Before proving the proposition, we note that the projection $\Pi[\U\U^T]$ is only needed in order to handle a degenerate case in which the dimension of the subspace spanned by the targets is smaller than $r$. Typically, this projection is not needed and $\U\U^T$ is feasible. 

\begin{proof}
We rely on the observation that F-FPCA has an optimal solution with orthogonal matrix, and for orthogonal matrix we have:
$$-f_i(\U)=\x_i^T\U\U^T\x_{i}$$
In addition the function $\U\mapsto\U\U^T$ is a surjective function from $\orth{d}{r}$ to $\mathcal{P}_r\setminus\mathcal{P}_{r-1}$, so the optimization over both sets is equivalent. More details are available in the Appendix.
\end{proof}

The advantage of solving F-FPCA rather than FPCA is that it \gnote{forces a low rank solution via} an unconstrained optimization. A member of the sub-gradient of F-FPCA objective can be computed in $O(d r n)$. In particular, Algorithm 1 describes a promising sub-gradient descent method for its minimization.

The obvious downside of using F-FPCA is its non-convexity that may cause descent algorithms to converge to bad stationary points. \gnote{Its convergence analysis is more difficult due to the non-smooth maximum function. Nonetheless, in the next section, we prove that there are no bad local minima when the targets are orthogonal. This is also demonstrated in the experimental  section where we show the advantages of F-FPCA in terms of accuracy.}

{\bf{Relation to other low rank optimization papers:}}
We note in passing that there is a large body of literature on global optimality properties of low rank optimizations \cite{zhu2018global,li2018non}. These provide sufficient conditions for convergence to global optimum in factorized formulations, e.g., Restricted Strong Convexity and Smoothness. Observe that these guarantees require the existence of a low rank optimal solution in the original problem. These conditions do not hold in FPCA, and therefore our analysis below takes a different approach.  

\begin{algorithm}[H]
 \caption{F-FPCA via sub-gradient descent}
 \begin{algorithmic}[1]
 \renewcommand{\algorithmicrequire}{\textbf{Input: }}
 \renewcommand{\algorithmicensure}{\textbf{Output:}}
 \REQUIRE $\{\x_i\}_{i=1}^n\subset\mathbb{R}^d$, $r\in\mathbb{N}$, $\eta$.
 \ENSURE $\P\in\mathcal{P}^d, {\rm{rank}}\left({\bf{P}}\right)\le r$.
  \STATE $t\leftarrow 0$ 
  \STATE draw $\U\in\mat{d}{r}$ randomly.
  \REPEAT 
  \STATE $t\leftarrow t + 1$
  \STATE $\hat{i}\leftarrow\arg\max_{{i\in[n]}}\norm{\x_{{i}} - \U\U^T\x_{{i}}}^2 - \norm{\x_{i}}^2$
  \STATE $\U\leftarrow\U-\frac{\eta}{t}\left(\x_{\hat{i}}\x_{\hat{i}}^T\U\U^{T}+\U\U^{T}\x_{\hat{i}}\x_{\hat{i}}^T-2\x_{\hat{i}}\x_{\hat{i}}^T\right)\U$
  \UNTIL convergence
 \RETURN $\P=\Pi\left[\U\U^T\right]$
 \end{algorithmic}
 \end{algorithm}

\section{Analysis - the orthogonal case}
In this section, we analyze the FPCA in the special case of orthogonal targets. As explained, FPCA is NP-hard and we do not expect a scalable and accurate solution for arbitray targets. Interestingly, our analysis shows that the problem becomes significantly easier when the targets are orthogonal. This is the case for example when the targets are randomly generated and the number of targets is much smaller than their dimension.

We will use the following assumptions:
\begin{enumerate}[label={\bf{A\arabic*}}:]
    \item The targets $\{\x_i\}_{i=1}^n\subset \mathbb{R}^d$ are orthogonal vectors.
    \item The problem is not degenerate in the sense that $$\frac{r}{n}H< \min_i\norm{\x_i}^2$$ where $$H=\frac{n}{\sum_{i=1}^n \frac{1}{\left\Vert\x_{i}\right\Vert^{2}}}$$ 
    (the harmonic mean of the squared norms of $\{\x_i\}^n_{i=1}$).
\end{enumerate}
Assumption {\bf{A1}} is the main property that simplifies the landscape and allows a tractable solution and analysis. On the other hand, assumption {\bf{A2}} is a technical condition that prevents a trivial degenerate solution based on the norms of the targets.   


Using these assumptions, we have the following results.

\begin{proposition} \label{NLM}
    Under assumptions {\bf{A1}}-{\bf{A2}}, 
    any local minimizer of F-FPCA is a global maximizer of FPCA and FPCA$=\frac{r}{n}H$.
\end{proposition}

\begin{proof}  
The proof consists of the following lemmas (proofs in the appendix): 

\begin{lemma} \label{local optimal is orthogonal}
Under assumptions \textbf{A1-A2}, let $\U\in\mat{r}{d}$ be a local minimizer of F-FPCA, then $\U\in\orth{d}{r}$.
\end{lemma}

\begin{lemma} \label{AE}
Under assumptions \textbf{A1-A2}, let $\U\in\orth{d}{r}$ a local minimizer of F-FPCA, then:
$f=f_i(\U)=f_j(\U)$ for all $i,j\in[n]$.
\end{lemma}

Intuitively, if the property in Lemma \ref{AE} is violated, then $\U$ can be infinitesimally changed in a manner that decreases the correlation of $\U$ with some direction $\w$ such that $\w\perp\x_j$ for all $j\in\mathcal{I}_\U$. We can decrease the value of $f_i$ for some $i\in\mathcal{I}_\U$ without harming the objective function using a sequence of Givens rotations with respect to the pairs $\{\w,\x_i\}$ for each $i\in\mathcal{I}_\U$. After decreasing $f_i$ for all $i\in\mathcal{I}_\U$ the objective will also be decreased.


Finally, in order to prove global optimality we define:
\begin{align} \label{target axes}
    \X=\left(\frac{\x_1}{\norm{\x_1}},...,\frac{\x_n}{\norm{\x_n}}\right),\quad\hat{\U}=\X^T\U
\end{align}
If $f=f_i(\U)=f_j(\U)$:
$$
    \left\Vert \hat{\U}_{i:}\right\Vert^2
    =-\frac{f_i(\U)}{\left\Vert\x_{i}\right\Vert^{2}}
    =-\frac{f}{\left\Vert\x_{i}\right\Vert^{2}}
$$
We have:
\begin{align*}
         r
    =\left\Vert \hat{\U}\right\Vert^2_F
    =\sum_{i=1}^n \left\Vert \hat{\U}_{i:}\right\Vert^2
    =-\sum_{i=1}^n \frac{f}{\left\Vert\x_{i}\right\Vert^{2}} 
\end{align*}
Rearranging yields $f= -\frac{r}{n}H$. 
Together with the equivalence in Proposition \ref{FPCA equivalent to F-FPCA} we conclude that all local minima yield an identical objective of $\frac{r}{n}H$ which is globally optimal.
\end{proof}

Proposition \ref{NLM} justifies the use of Algorithm 1 when the targets are orthogonal. Numerical results in the next section suggest that bad local minima are rare even in more realistic near-orthogonal scenarios. 

Given the favourable properties of F-FPCA in the orthogonal case, it is interesting to analyze the performance of SDR in this case. 
\begin{proposition} \label{MT}
    Under assumptions {\bf{A1}}-{\bf{A2}}, SDR is tight and its optimal objective value is $${\rm{SDR}}=\frac{r}{n}H.$$ However, the optimal solution may be full rank and infeasible for FPCA.
\end{proposition}
\begin{proof}
See Appendix.
\end{proof}

\gnote{Observe that the rank constraint is hard, and a rank reduction procedure such as PrSDR is necessary for finding a feasible solution. 
The iterative rounding algorithm of \cite{morgenstern2019fair} relies on finding an extreme point solution, and guarantees an upper bound on its rank. The bound is not always tight. For example, their algorithm fails to find an optimal low rank solution in the orthogonal case.}
On the other hand, Algorithm 1 easily finds the global solution.

Finally, we conclude this section by noting an interesting relation between FPCA with orthogonal targets and the design of Finite Tight Frames \cite{benedetto2003finite}. Recall the following definition:
\begin{definition}.
\begin{itemize}
    \item Let $\{\u_i\}_{i=1}^n\subset\mathbb{R}^r$. If $span(\{\u_i\}_{i=1}^n)=\mathbb{R}^r$ then $\{\u_i\}_{i=1}^n$ is frame for $\mathbb{R}^r$.
    \item A frame $\{\u_i\}_{i=1}^n$ is tight with frame bound A if $\forall \v\in\mathbb{R}^n$:
    $$\v=\frac{1}{A}\sum_{i=1}^n\left<\v,\u_i\right>\u_i$$
    \item A frame $\{\u_i\}_{i=1}^n$ is a 'Normalized Tight Frame' if $\{\u_i\}_{i=1}^n$ is tight frame and $\Vert \u_i\Vert=1$ for all $i$.
\end{itemize}
\end{definition}
A straight forward consequence is the following result.
\begin{corollary}
Under assumptions {\bf{A1}}-{\bf{A2}}, if $\U$ is an optimal solution for F-FPCA, then $\U^T$ is a tight frame. In particular, if the targets are the standard basis, then $\frac{d}{r}\U^T$ is a normalized tight frame.
\end{corollary}

{\em{Sketch of proof (the proof in the appendix):}} As proved before, the solution of F-FPCA is in $\orth{d}{r}$ and the transposition of any $\U\in\orth{d}{r}$ is a tight frame. The second part is true since the optimal solution of F-FPCA is satisfied for all k: $\left\Vert \x_k^T\U\right\Vert^2=\frac{r}{n}H$. For the standard basis we get for all $i,j$: $\norm{\U^T\e_i}=\norm{\U^T\e_j}$ i.e. the norm of all rows of $\U$ are equals.  

It is well known that normalized tight frames can be derived as minimizers of frame potential functions \cite{benedetto2003finite}. The corollary provides an alternative derivation via FPCA with different targets $\x_i$. Depending on the properties of the targets, this allows a flexible data-driven design that will be pursued in future work.

\section{Experimental results}

In this section, we illustrate the efficacy of the different algorithms using numerical experiments. We compare the following competitors:
\begin{itemize}
    \item SDR - a (possibly infeasible) upper bound defined as the solution to (\ref{SDR}) via CVXPY \cite{cvxpy,cvxpy_rewriting}.
    \item PIRSDR - the projection of SDR onto the feasible set using eigenvalue decomposition. \gnote{Before project the solution, the iterative rounding rank reduction from \cite{morgenstern2019fair} was performed.}
    \item F-FPCA - the solution to (\ref{factorized_restore_problem}) via Algorithm 1 with a random initialization.
    \item F-FPCAi -  the solution to (\ref{factorized_restore_problem}) via Algorithm 1 with PIRSDR initialization.
\end{itemize}
To allow easy comparison, we normalize the results by the value of SDR, so that a ratio of $1$ corresponds to a tight solution.

\subsection{Synthetic simulations}
We begin with experiments on synthetic targets with independent, zero mean, unit variance, Gaussian random variables. This is clearly a simplistic setting but it allows control over the different parameters $r$, $n$ and $d$. Each of the experiments was performed $15$ times and we report the average performance.

Rank effect: The first experiment is presented in Fig. 1 and illustrates the dependency on the rank $r$. It is easy to see that even with very small rank, the gap between the upper and lower bounds vanishes. We conclude that in this non-orthogonal setting, the landscape of FPCA is benign as long as the rank is not very small.


\begin{figure}[htbp]
\centerline{\includegraphics[width=0.5\textwidth]{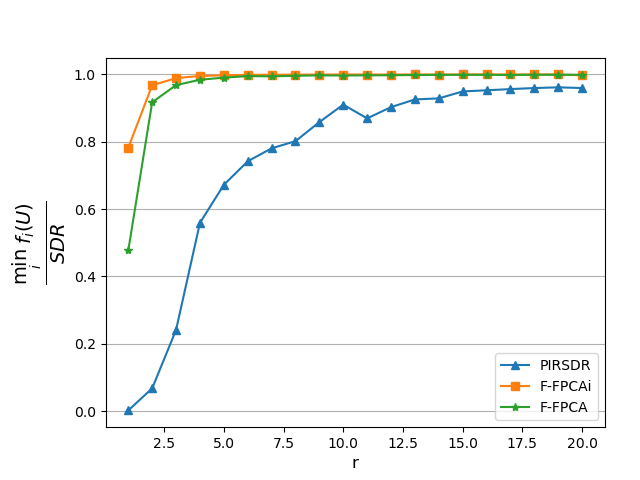}}
\caption{Quality of approximation as a function of the rank ($d=200,n=50$)}
\end{figure}


Orthogonality effect: The second experiment is presented in Fig. 2 and addresses the effect of orthogonality. As explained, the targets are drawn randomly and they tend to orthogonality as $d$ increases. Our analysis proved that the gap should vanish when the targets are exactly orthogonal. The numerical results suggest that this is also true for more realistic and near-orthogonal targets. The optimality gap clearly decreases as $d$ increases.

\begin{figure}[htbp]
\centerline{\includegraphics[width=0.5\textwidth]{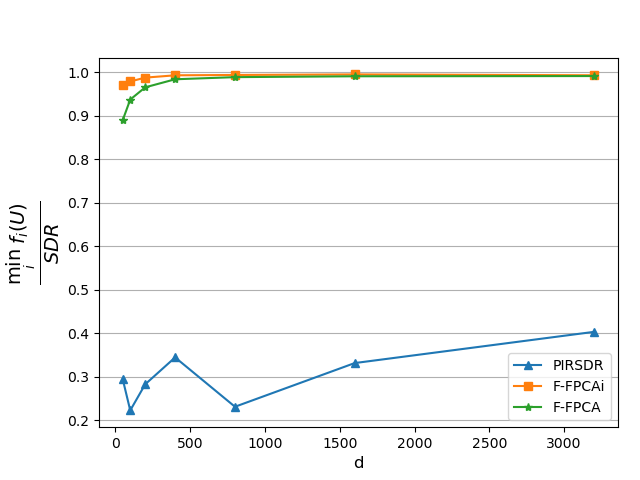}}
\caption{Quality of approximation as a function of the orthogonality ($n=50,r=2$)}
\end{figure}



\subsection{Minerals dataset}
In order to illustrate the performance in a more realistic setting we also considered a real world dataset. We consider the design of hyperspectral matched filters for detection of known minerals. We downloaded spectral signatures of minerals from the \href{https://crustal.usgs.gov/speclab/QueryAll07a.php?page=1}{Spectral Library of United States Geological Survey (USGS)}. We experimented with $114$ different minerals, each with $480$ bands in the range $0.01\mu-3\mu$. Some of the measurements were missing and their corresponding bands were omitted. We then performed PCA and reduced the dimension from $421$ to $\mathbb{R}^{114}$. These vectors were normalized and then centered. Fig. 3 provides the signatures of the first minerals before and after the pre-processing. Finally, we performed fair dimension reduction to $r= 1...6$. Fig. 4 summarizes the quality of the approximation of the different algorithms. As before, it is easy to see that F-FPCA is near optimal at very small ranks. Interestingly, PIRSDR is beneficial as an initialization but shows inferior and non-monotonic performance on its own. \gnote{As expected, all the algorithms easily attain the optimal performance at higher values of $r$.}


\begin{figure}[htbp]
\centerline{\includegraphics[width=0.5\textwidth]{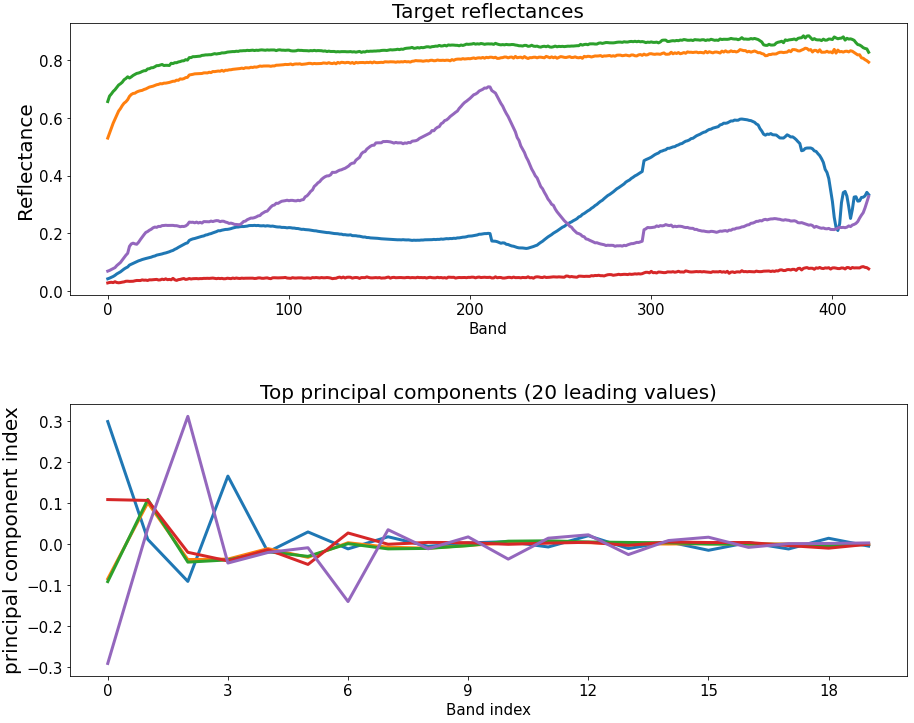}}
\caption{The spectral signatures of Actinolite, Adularia, Albite, Allanite and Almandine.}
\end{figure}

\begin{figure}[htbp]
\centerline{\includegraphics[width=0.5\textwidth]{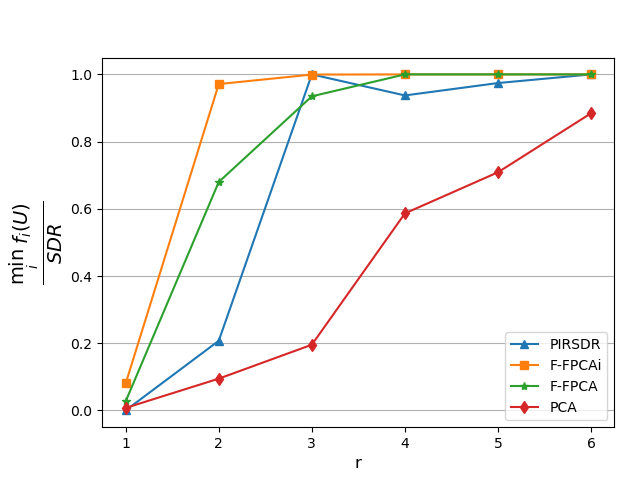}}
\caption{Quality of approximation as a function of the rank: minerals dataset.}
\end{figure}

\gnote{
\subsection{Credit dataset}
Next, we continue to the credit dataset \cite{yeh2009comparisons} that was considered for block-FPCA in  \cite{olfat2019convex,morgenstern2019fair,samadi2018price}. Following these works, we consider the functions
\begin{equation}
f_i(\U)=\norm{\X_{i} - \U\U^T\X_{i}}^2 - PCA(\X_i)
\end{equation}
where $\X_i=\sum_{j=1}^{n_i}\x_j^i\x_j^{iT}$ and $PCA$ is the objective of the standard PCA function of $\X_i$ which is independent of $\U$. 
The results in Fig 6 are identical to those in \cite{morgenstern2019fair} with our additional F-FPCA algorithm. PIRSDR achieves the SDR lower bound at all ranks excepts $r=7$. Remarkably, F-FPCA is optimal in this specific setting and attains the bound for all $r$ without exception. Apparently, the landscape of the credit dataset is benign. We emphasize that this is pure luck and we can easily find other non-orthogonal examples with spurious local minima. In real applications, we recommend running F-FPCA using multiple initializations and choosing the best solution.}
\begin{figure}[htbp] \label{matrix6}
\centerline{\includegraphics[width=0.5\textwidth]{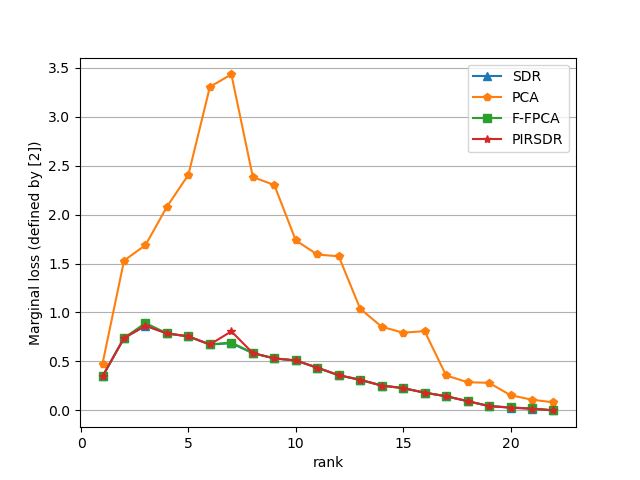}}
\caption{Marginal loss from \cite{morgenstern2019fair} when splitting the credit data to 6 classes}
\end{figure}

\section{Conclusion}
In this paper, we suggested to tackle the problem of fairness in linear dimension reduction by simply using first order methods over a non-convex optimization problem. We provided an analysis of the landscape of this problem in the special case where the targets are orthogonal to each other. We also provided experimental results which support our approach by showing that sub gradient descent is successful also in the near orthogonal case and real world data. 

There are many interesting extensions to this paper that are worth pursuing. Analysis of the near-orthogonal case is still an open question. In addition, a drawback of our approach is the non smoothness of the landscape which might prevent the use of standard convergence bounds for first order methods. This can be treated by approximating the $L_{2,\infty}$ in our formulation by log-sum-exp or $L_{2,p}$ norm for $p<\infty$ functions. Experimental results show that our results can be extended to the block case that is more relevant to machine learning. Finally, we only considered the case of classical linear dimension reduction. Future work may focus on extensions to non-linear methods and tensor decompositions.

\appendices

\section{Proof of Proposition 1}

Let $\O\D\O^{T}$ a truncated EVD decomposition of $\U\U^{T}$, then:
\begin{align*}
    &\left\Vert\x_{i}-\U\U^{T}\x_{i}\right\Vert^{2} - \left\Vert\x_{i}\right\Vert^{2}\\
    =&\left\Vert\x_{i}\right\Vert^{2}+\x_{i}^T\U\U^T\U\U^{T}\x_{i}-2\x_{i}^T\U\U^{T}\x_{i}-\left\Vert\x_{i}\right\Vert^{2}\\
    =&\x_{i}^T\O\D^2\O^{T}\x_{i}-2\x_{i}^T\O\D\O^{T}\x_{i}\\
    =&\x_{i}^T\O(\D^2-2\D)\O^{T}\x_{i}\\
    =&\sum_{l=1}^{r}(\D_{ll}^{2}-2\D_{ll})\left<\O_{:l},\x_i\right>^{2}
\end{align*}

Observe that this function is minimized when $\D_{ll}=1$ for all $l\le r$, so:
\begin{align*}
    \left\Vert\x_{i}-\U\U^{T}\x_{i}\right\Vert^{2} - \left\Vert\x_{i}\right\Vert^{2}\ge-\x_i^T\Pi[\U\U^T]\x_i
\end{align*}

So F-FPCA is equivalent to the following problem (over the orthogonal matrices):
\begin{align*}
\begin{array}{ll}
     \min_{\U\in\orth{d}{r}}  & \max_{{i\in[n]}}\quad \left\Vert\x_{i}-\U\U^{T}\x_{i}\right\Vert^{2}-\left\Vert\x_{i}\right\Vert^{2}
\end{array}
\end{align*}

Now for any orthogonal matrix $\U$ we get:
\begin{align*}
\begin{array}{ll}
     = \left\Vert\x_{i}-\U\U^{T}\x_{i}\right\Vert^{2}-\left\Vert\x_{i}\right\Vert^{2}\\
     = \left\Vert\x_{i}\right\Vert^{2}+\x_{i}^T\U\U^T\U\U^{T}\x_{i}-2\x_{i}^T\U\U^{T}\x_{i}-\left\Vert\x_{i}\right\Vert^{2}\\
     = \x_{i}^T\U\U^{T}\x_{i}-2\x_{i}^T\U\U^{T}\x_{i}\\
     = -\x_{i}^T\U\U^{T}\x_{i}\\
\end{array}
\end{align*}
\\
Finally, observe that:
\begin{itemize}
    \item $\U$ is a feasible solution for the problem above iff $\P=\U\U^T$ is a feasible solution for FPCA.
    \item The objective function of FPCA in $\P$ is equal to the objective function of the problem above in $\U$ (multiplied by $-1$).
\end{itemize}
So we conclude that the problems are equivalent.

\section{Proof of \Lemmaref{local optimal is orthogonal}}
We begin with the following lemma:
\begin{lemma} \label{Necessary condition to A2}
Let $\U\in\orth{d}{r}$. If \textbf{A2} holds then $\forall i \in \mathcal{I}_\U: \quad \norm{\x_i}> \norm{\U\U^T\x_i}$. 
\end{lemma}
\begin{proof}
Assume in contradiction that there exists $k\in\mathcal{I}_\U$ ($\mathcal{I}_\U:=\arg\max_if_i(\U)$) such that: $\norm{\x_k}= \norm{\U\U^T\x_k}$, and let $j\in\arg\min_{\x_i}\norm{\x_i}$. We get for all $i$:
$$
f_i(\U)\le f_k(\U)\gnote{\le}-\norm{\x_k}^2\le-\norm{\x_j}^2
$$
Now recall the definition of $\hat{\U}$ in (\ref{target axes}) and observe that:
\begin{align*}
    &r=\left\Vert \hat{\U}\right\Vert^2_F\gnote{=}\sum_{i=1}^n \left\Vert\hat{\U}_{i:}\right\Vert^2=\sum_{i=1}^n\frac{-f_i(\U)}{\left\Vert\x_{i}\right\Vert^{2}}
    \ge \sum_{i=1}^n\frac{\left\Vert\x_{j}\right\Vert^{2}}{\left\Vert\x_{i}\right\Vert^{2}}\\
    & \Rightarrow \norm{\x_j}^2\le \frac{r}{\sum_{i=1}^n \frac{1}{\left\Vert\x_{i}\right\Vert^{2}}}
\end{align*}
This means that \textbf{A2} does not hold.
\end{proof}

We will now show that if $\U$ is not orthogonal, then we can decrease either the size of $\mathcal{I}_\U$ or the value of $\max_if_i(\U)$ by choosing an arbitrarily close $\U'$.

\begin{lemma} \label{not optimal can reduce I}
Let $\U\notin\orth{d}{r}$, then
for any $\epsilon>0$ there exists a $\U'$ such that:
\begin{enumerate}
    \item $\left\Vert \U-\U'\right\Vert \le\epsilon$.
    \item Either $|\mathcal{I}_{\U}|>|\mathcal{I}_{\U'}|$, or $\max_if_i(\U)>\max_if_i(\U')$.
\end{enumerate}
\end{lemma}
\begin{proof}
Let $\O\D\O^{T}$ an EVD decomposition of $\U\U^{T}$, then:
\begin{align*}
    \left\Vert\x_{i}-\U\U^{T}\x_{i}\right\Vert^{2} - \left\Vert\x_{i}\right\Vert^{2}=\sum_{l=1}^{r}(\D_{ll}^{2}-2\D_{ll})\left<\O_{:l},\x_i\right>^{2}
\end{align*}
Due to $\U\notin\orth{d}{r}$, there is an $\hat{l}\le r$ such that ${\D}_{\hat{l},\hat{l}}\ne1$, and an $i$ such that $\left<\O_{:\hat{l}},\x_i\right>\ne0$. Observe that: $h_i({\D}_{\hat{l},\hat{l}})=({\D}_{\hat{l},\hat{l}}^2-2{\D}_{\hat{l},\hat{l}})\left<\O_{:\hat{l}},\x_i\right>^{2}$ has a local minimum only in ${\D}_{\hat{l},\hat{l}}=1$. Therefore, define $\U^{'}=\O\D'\O^T$ where:
$$
{\D}_{\hat{l},\hat{l}}'=\begin{cases}
{\D}_{\hat{l},\hat{l}}'={\D}_{\hat{l},\hat{l}}-\epsilon&\quad{\D}_{\hat{l},\hat{l}}>1\\
{\D}_{\hat{l},\hat{l}}'={\D}_{\hat{l},\hat{l}}+\epsilon&\quad{\D}_{\hat{l},\hat{l}}<1
\end{cases}
$$
and we get $f_j(\U)>f_j(\U')$ for all $j$ such that $\left<\O_{:\hat{l}},\x_j\right>^{2}\ne0$. 

If there exists an $\hat{l}$ such that ${\D}_{\hat{l},\hat{l}}\ne1$ and  $\left|\left\{j|\left<\O_{:\hat{l}},\x_j\right>^{2}\ne0\right\}\cap\mathcal{I}_\U\right|>0$ then we are done.

Otherwise, pick some $\hat{l}$ with ${\D}_{\hat{l},\hat{l}}\ne1$, and after the procedure above take $\x_k\in\mathcal{I}_\U$ and define $\x^\perp_k$ the projection of $\x_k$ onto ${\rm{Im}}(\U)^\perp$  (by \Lemmaref{Necessary condition to A2} $\x_k^\perp\ne0$). Define $\O'$ by adding $\epsilon\x^\perp_k$ to the $\hat{l}'th$ singular vector $\O_{:\hat{l}}$ of $\U'$ and define $\U''=\O'\D'\O'^T$. Now we get that for all $i\in\mathcal{I}_{\U'}$:
\begin{align*}
    &\left<\O_{:\hat{l}}',\x_i\right>^2=\left(\left<\O_{:\hat{l}}+\epsilon\x_k^\perp,\x_i\right>\right)^2\\
    =&\left(\left<\O_{:\hat{l}},\x_i\right>+\epsilon\left<\x_k^\perp,\x_i\right>\right)^2=\epsilon^2\left<\x_k^\perp,\x_i\right>^2\ge0\\
    \Rightarrow &\left\Vert\x_{i}-\U''\U''^{T}\x_{i}\right\Vert^{2}-\left\Vert\x_{i}\right\Vert^{2}=
    \sum_{l=1}^{r}(\D_{ll}'^{2}-2\D_{ll}')\left<\O_{:l}',\x_i\right>^{2}\\
    \le&\sum_{l=1}^{r}(\D_{ll}'^{2}-2\D_{ll}')\left<\O_{:l},\x_i\right>^{2}
\end{align*}
Similarly, for $\x_k$ we get:
\begin{align*}
    &\sum_{l=1}^{r}(\D_{ll}'^{2}-2\D_{ll}')\left<\O_{:l}',\x_k\right>^{2}\\
    =&\sum_{l=1}^{r}(\D_{ll}'^{2}-2\D_{ll}')\left<\O_{:l},\x_k\right>^{2}+\epsilon(\D_{\hat{l}\hat{l}}'^{2}-2\D_{\hat{l}\hat{l}}')\left<\x_k^\perp,\x_k\right>^{2}\\
    <&\sum_{l=1}^{r}(\D_{ll}'^{2}-2\D_{ll}')\left<\O_{:l},\x_k\right>^{2}
\end{align*}
as required.
\end{proof}

We can now apply \Lemmaref{not optimal can reduce I} iteratively as follows. Let $\U\notin\orth{d}{r}$, and let $\epsilon>0$. By \Lemmaref{not optimal can reduce I}: 
\begin{itemize}
    \item There is $\U_{1}$ with $\left\Vert \U_{1}-\U\right\Vert\le\frac{\epsilon}{n}$, s.t.: $|\mathcal{I}_{\U}|>|\mathcal{I}_{\U_1}|$.
    \item There is $\U_{2}$ with $\left\Vert \U_{2}-\U_{1}\right\Vert\le\frac{\epsilon}{n}$, s.t.: $|\mathcal{I}_{\U_1}|>|\mathcal{I}_{\U_2}|$.
    \item ...
    \item There is $\U'$ with $\left\Vert \U'-\U_{K}\right\Vert\le\frac{\epsilon}{n}$, s.t.: $\max_if_i(\U_{K})>\max_if_i(\U')$.
\end{itemize}
Finally, observe that $K + 1\le |\mathcal{I}_U|$, so $\left\Vert \U-\U'\right\Vert\le \epsilon\frac{K+1}{n}\le\epsilon$ and we can find arbitrarily close $\U'$ such that $\max_if_i(\U')<\max_if_i(\U)$ i.e. $\U$ is not a local minimizer.

\section{Proof of \Lemmaref{AE} }

We begin with the following lemma that states that we can utilize the orthogonality of the targets in order to infinitesimally change $\U$ in a manner that increases the value of $f_j$ for some $j\notin\mathcal{I}_\U$, decreases the value of $f_i$ for some $i\in\mathcal{I}_\U$ and does not change the value of $f_k$ for $k\in\mathcal{I}_{\U}\setminus \{i\}$.

\begin{lemma} \label{not equal implies can improve} 
Let $\U\in\orth{d}{r}$ such that there exist $j$ with $f_j(\U)<\max_{l\in[n]}f_l(\U)$. Then, there exists an $i\in\mathcal{I}_\U$ such that:
for any $\epsilon>0$ there exist $\U_{\theta}$ such that:
\begin{enumerate}
    \item $\left\Vert \U-\U_{\theta}\right\Vert \le\epsilon$.
    \item $\forall k\in [n]\setminus\mathcal{I}_\U:\quad f_k(\U_{\theta})<f_i(\U)$
    \item $\forall k\in\mathcal{I}_\U\setminus\left\{ i\right\} :\qquad f_k(\U)=f_k(\U_{\theta})$.
    \item $f_i(\U_{\theta})<f_i(\U)$
\end{enumerate}
\end{lemma}

\begin{proof}
Define $\Rt$, a Given Rotation (for some $\theta$) over the $1,2$ axes, i.e.:
$$
    \Rt_{ij}=\begin{cases}
    \cos\theta & ij=11\quad or\quad ij=22\\
    \sin\theta & ij=12\\
    -\sin\theta & ij=21\\
    \I_{ij} & else
    \end{cases}
$$
along with two orthogonal vectors
\begin{eqnarray}
 \v_1 &=& \frac{\x_{i}}{\norm{\x_{i}}}\nonumber\\
  \v_2 &=&  \begin{array}{lcr}
     
        \begin{cases}
            \frac{\x_{j}}{\norm{\x_{j}}} & \exists i\in\mathcal{I}_\U:\; \x_i^T\U\U^T\x_{j}\ne0\\
            \frac{\U\U^T\x_{j}}{\norm{\U\U^T\x_{j}}} & {\rm{else}}
        \end{cases}
    \end{array}
\end{eqnarray}
and \gnote{an orthonormal basis for their orthogonal complement in $\mathbb{R}^d$: } $\V=\left(\v_1,...\v_d\right)$. Now define: $\U_\theta=\V\Rt\V^T\U$ and we get:\\

\textbf{1+2 is true, since:} 

$h_1\left(\theta\right):=\U_{\theta},\quad h_2(\theta):=f_j(\U_{\theta})$ are continuous functions.\\

\textbf{3 is true, since:} 

For all $k\in\mathcal{I}_U\setminus\left\{i\right\}:\x_k\perp\v_1,\v_2$ thus $\Rt\V^T\x_k=\I\V^T\x_k$ and:
\begin{align*}
    \x_{k}^T\U_{\theta}\U_{\theta}^{T}\x_{k}
    =&\x_{k}^T\V\Rt\V^T\U\U^{T}\V\Rt^{T}\V^T\x_{k}\\
    =&\x_{k}^T\V\I\V^T\U\U^{T}\V\I\V^T\x_{k}\\
    =&\x_{k}^T\U\U^{T}\x_{k}\\
\end{align*}

\textbf{In order to show 4 we use the equality in (\ref{diff_fU})} (\gnote{In the next page}, proof is in the appendix, since it is quite technical):

\begin{figure*}
\begin{align}\label{diff_fU}
    f_i(\U_{\theta})-f_i(\U)=\begin{cases}
\sin\left(2\theta\right)\e_{i}^{T}\hat{\U}\hat{\U}^{T}\e_{j}+\sin^{2}\left(\theta\right)\left(\e_{j}^T\hat{\U}\hat{\U}^{T}\e_{j}-\e_{i}^T\hat{\U}\hat{\U}^{T}\e_{i}\right) & \exists i\in\mathcal{I}_\U:\; \x_i^T\U\U^T\x_{j}\ne0\\
(\norm{\x_{i}}^2-\norm{\U^T\x_{i}}^2){\sin(\theta)^2} & else
\end{cases}
\end{align} 
\end{figure*}

Now, if $\exists i\in\mathcal{I}_\U:\; \x_i^T\U\U^T\x_{j}\ne0$:
\begin{itemize}
    \item If: $\x_i^T\U\U^T\x_{j} < 0$ then for any $\frac{\pi}{2}>\theta>0$:  $\sin\left(2\theta\right)\e_{i}^{T}\hat{\U}\hat{\U}^{T}\e_{j}<0$.
    \item If: $\x_i^T\U\U^T\x_{j} > 0$ then for any $-\frac{\pi}{2}<\theta<0$:  $\sin\left(2\theta\right)\e_{i}^{T}\hat{\U}\hat{\U}^{T}\e_{j}<0$.
\end{itemize}
and since $\sin(\theta)^2=o(\sin(2\theta))$, for small enough $|\theta|$ we get: $f_i(\U)<f_i(\U_{\theta})$.

On the other hand, By \Lemmaref{Necessary condition to A2}  $\norm{\x_{i}}^2>\norm{\U^T\x_{i}}^2$, so if $\forall i\in\mathcal{I}_\U:\; \x_i^T\U\U^T\x_{j}=0$ then: $$(\norm{\U^T\x_{i}}^2-\norm{\x_{i}}^2){\sin(\theta)^2}<0$$ 

\end{proof}

Armed with these results, we proceed to the rest of the proof of \Lemmaref{AE}. Assume $f_i(\U)<\max_lf_l(\U)$ for some $i$. By \Lemmaref{not equal implies can improve}, let $\epsilon>0$, then: 
\begin{itemize}
    \item There is $\U_{1}$ with $\left\Vert \U_{1}-\U\right\Vert\le\frac{\epsilon}{K}$, s.t.: $\max_if_i(\U_1)=\max_if_i(\U)$, and  $|\mathcal{I}_{\U_1}|=|\mathcal{I}_{\U}|-1$. 
    \item ...
    \item There is $\U_{K}$ with $\left\Vert \U_{K}-\U_{K-1}\right\Vert\le\frac{\epsilon}{K}$, s.t.: $\max_if_i(\U_{K})<\max_if_i(\U_{K-1})$.
\end{itemize}
Finally, observe that $\left\Vert \U-\U_{K}\right\Vert\le{\epsilon}$  \gnote{and $K\le|\mathcal{I}_U|$,} so we can find arbitrarily close $\U_K$ such that $\max_if_i(\U_{K})<\max_if_i(\U)$ i.e. $\U$ is not a local maximizer.

\section{Proof of \propref{MT}}

\begin{proof}
Given $\{\x_i\}_{i=1}^n$, and recall the definition of $\X$ in (\ref{target axes}). In the SDR problem we solve:
\begin{align*}
    &\begin{array}{ll}
        \max_{{\bf{P}}\in\mathbb{S}^d_+} & \min_{i\in[n]} \x_{i}^{T}{\bf{P}}\x_{i}\\
        {\rm{s.t.}} &   {\rm{Trace}}\left({\bf{P}}\right)\le r\\
             &   {\bf{0}}\preceq {\bf{P}}\preceq {\bf{I}}
    \end{array}\\
    =&\begin{array}{ll}
        \max_{{\bf{P}}\in\mathbb{S}^d_+} & \min_{i\in[n]} \left\Vert \x_{i}\right\Vert^{2}\e_{i}^T\X\P\X^T\e_{i}\\
        {\rm{s.t.}} &   {\rm{Trace}}\left({\bf{P}}\right)\le r\\
             &   {\bf{0}}\preceq {\bf{P}}\preceq {\bf{I}}
    \end{array}\\
    =&\begin{array}{ll}
        \max_{{\bf{P}}\in\mathbb{S}^d_+} & \min_{i\in[n]} \left\Vert \x_{i}\right\Vert^{2}[\X\P\X^T]_{ii}\\
        {\rm{s.t.}} &   {\rm{Trace}}\left(\X\P\X^T\right)\le r\\
             &   {\bf{0}}\preceq \X\P\X^T\preceq {\bf{I}}
    \end{array}\\
    =&\begin{array}{ll}
        \max_{\hat{\P}\in\mathbb{S}^d_+} & \min_{i\in[n]} \left\Vert \x_{i}\right\Vert^{2}\hat{\P}_{ii}\\
        {\rm{s.t.}} &   {\rm{Trace}}\left(\hat{\P}\right)\le r\\
             &   {\bf{0}}\preceq \hat{\P}\preceq {\bf{I}}
    \end{array}
\end{align*}
where $\X\P\X^T=\hat{\P}$ and we have used the orthogonality of $\X$. 

Now, define $\hat{\P}$ as a diagonal matrix with
$
\hat{\P}_{ii}=\frac{r}{n}\cdot\frac{H}{\norm{\x_{i}}^{2}}
$.

It is easy to verify that $\hat{\P}$ is feasible and yields an objective of $\frac{r}{n}H$. Now let $\hat{\P}'$ such that $\min_k\left\Vert\x_{k}\right\Vert^{2}\hat{\P}_{kk}'>\frac{r}{n}H$, then:
\begin{align*}
     &\forall i: \left\Vert\x_{i}\right\Vert^{2}\hat{\P}_{ii}'> \frac{r}{n}H\\ \Rightarrow& \forall i:\hat{\P}_{ii}'>\frac{\frac{r}{n}H}{\left\Vert\x_{i}\right\Vert^{2}}\\
    \Rightarrow &{\rm{Trace}}(\hat{\P}')=\sum_{i=1}^n\hat{\P}_{ii}'>\sum_{i=1}^n\frac{\frac{r}{n}H}{\left\Vert\x_{i}\right\Vert^{2}}=r
\end{align*}
so $\hat{\P}'$ is not feasible and we conclude that $\hat{\P}$ is optimal for SDR. By \propref{NLM} the optimal value for F-FPCA is $-\frac{r}{n}H$ so by \propref{FPCA equivalent to F-FPCA} the value of FPCA is $\frac{r}{n}H$ so SDR=FPCA (but this solution might not be low rank, as in our positive definite construction).
\end{proof}

\section{Proof of Corollary 1}

We start the proof of the proposition by the observation that tight frame is characterized by the standard basis:

\begin{lemma}\label{standard basis enough}
A frame $\{\u_i\}_{i=1}^n\subset\mathbb{R}^r$ is tight with frame bound A if and only if $\forall \e_i$ in the standard basis of $\R^r$:
    $$\e_i=\frac{1}{A}\sum_{i=1}^n\left<\e_i,\u_i\right>\u_i$$
\end{lemma}

\begin{proof}
Observe that if $\;\forall \e_{i}\in\left\{ \e_{i}\right\} _{i=1}^{r}$ we have: $\e_{i}=\frac{1}{A}\sum_{j=1}^{n}\left\langle \e_{i},\u_{j}\right\rangle \u_{j}$ than we get for all $\v\in\mathbb{R}^{r}$:
\begin{align*}
    &\v=	\sum_{j=1}^{r}\v_{j}\e_{j}
    =	\sum_{j=1}^{r}\v_{j}\frac{1}{A}\sum_{i=1}^{n}\left\langle \e_{j},\u_{i}\right\rangle \u_{i}\\
    =&	\frac{1}{A}\sum_{i=1}^{n}\left\langle \sum_{j=1}^{r}\v_{j}\e_{j},\u_{i}\right\rangle \u_{i}
    =	\frac{1}{A}\sum_{i=1}^{n}\left\langle \v,\u_{i}\right\rangle \u_{i}
\end{align*}
\end{proof}

Now we use the observation above to claim that tight frame is actual the transposition of semi orthogonal matrix:

\begin{lemma} \label{tight frame iff scaled orthonormal}
Let $\U^T=(\u_1,...,\u_n)\in\mathbb{R}^{r\times n}$. $\{\u_i\}_{i=1}^n$ is a tight frame with frame bound $A$ iff $\U$ has orthogonal columns with norm $\sqrt{A}$. 
\end{lemma}

\begin{proof}

Consider equality 1 below: 
\begin{align*}
    &\U^{T}\U=\left(\u_{1},...,\u_{n}\right)\left(\u_{1},...,\u_{n}\right)^{T}\\
    =&\left(\begin{array}{c}
    \\
    \sum_{j=1}^{n}\u_{j}^{1}\u_{j},\\
    \\
    \end{array}...,\begin{array}{c}
    \\
    \sum_{j=1}^{n}\u_{j}^{r}\u_{j}\\
    \\
    \end{array}\right)\\
    =^1&
    \left(\begin{array}{c}
    \\
    A\e_{1}\\
    \\
    \end{array}...\begin{array}{c}
    \\
    A\e_{r}\\
    \\
    \end{array}\right)=A\I
\end{align*}
Observe that $\U\U^T=A\I$ iff $\U$ has orthogonal columns with norm $\sqrt A$. Equality 1 also holds iff $A\e_{i}=\sum_{j=1}^{n}\u_{j}^{i}\u_{j}=\sum_{j=1}^{n}\left\langle \e_{i},\u_{j}\right\rangle \u_{j}$ which holds iff $\{\u_i\}_{i=1}^n$ is a tight frame with frame bound $A$ (by \Lemmaref{standard basis enough}), so we conclude that the conditions are equivalent.

\end{proof}

\begin{lemma} \label{sum is r}
 If for all k: $\left\Vert \x_k^T\U\right\Vert^2=\frac{r}{n}H$, then: $\norm{\U}^2=r$.
\end{lemma}

\begin{proof}
Recall the definition of $\hat{\U}$ in (\ref{target axes}) and observe that:
\begin{align*}
    &\left\Vert \U\right\Vert_F^2
    =\left\Vert \hat{\U}\right\Vert_F^2
    =\sum_{i=1}^n\norm{\e_i^T\hat{\U}}^2
    =\sum_{i=1}^n\frac{1}{\norm{{\x}_i}^2}\norm{{\x}_i^T{\U}}^2\\
    =&\sum_{i=1}^n\frac{1}{\norm{{\x}_i}^2}r\left({\sum_{i=1}^n \frac{1}{\left\Vert \x_{i}\right\Vert^{2}}}\right)^{-1}
    =r
\end{align*}
\end{proof}

Now, let $\U\in\mathbb{R}^{d\times r}$ an optimal solution for F-FPCA, by \Lemmaref{local optimal is orthogonal} the columns of $\U$ are orthonormal, so by \Lemmaref{tight frame iff scaled orthonormal} $\U$ is tight frame and by \propref{NLM} we have for all $k$:
$$
\norm{\x_k^T\U}^2=f_k(\U)=\frac{r}{n}H
$$
On the other hand, let $\U^T$ a tight frame as above, by \Lemmaref{tight frame iff scaled orthonormal} the columns of $\U$ are orthogonal and have the same norm. By \Lemmaref{sum is r} $\norm{\U}_F^2=r$ so the columns has unit norms, i.e. the columns are orthonormal and for all $i\in[n]$:
\begin{align*}
    f_i(\U)
    =-\x_{i}^T\U\U^T\x_{i}
    =-\frac{r}{n}H
\end{align*}
i.e. $\min_i\left\Vert \x_i-\U\U^T\x_i\right\Vert^2-\left\Vert \x_{i}\right\Vert^{2}=-\frac{r}{n}H$ which is the optimal target. Finally, if $\forall i:\; \x_i=\e_i$, then F-FPCA is reduced to the problem of finding 'normalized tight frame'.

\section*{Acknowledgment}
The authors would like to thank Uri Okon who initiated this research and defined the problem, as well as Gal Elidan. This work was partially supported by ISF grant 1339/15.

\ifCLASSOPTIONcaptionsoff
  \newpage
\fi

\bibliography{IEEEabrv,bib}
\bibliographystyle{IEEEtran}

\begin{figure*}
\section{Proof of equation (11)}

\begin{lemma}
Let $\U\in\orth{d}{r}$ and assume $\forall i\in\mathcal{I}_\U:\; \x_i^T\U\U^T\x_{j}=0$. Define: 
$$\v_2 = \frac{\U\U^T\x_{j}}{\norm{\U\U^T\x_{j}}},\;\;\v_1 = \frac{\x_{i}}{\norm{\x_{i}}}$$
for some $i\in\mathcal{I}_{\U}$ and complete these vectors to orthonormal basis: $\V=\left(\v_1,...\v_d\right)$, and define: $\U_\theta=\V\Rt\V^T\U$, then:
\begin{align*}
    f_i(\U)-f_i(\U_{\theta})=(\norm{\x_{i}}^2-\norm{\U^T\x_{i}}^2){\sin(\theta)^2}
\end{align*} 

\end{lemma}
\begin{proof}
\begin{align*}
    f_i(\U)-f_i(\U_{\theta})
    =&\x_{i}^T\U_{\theta}\U_{\theta}^{T}\x_{i}-\x_{i}^T\U\U^{T}\x_{i}\\
    =&\left(\U_{\theta}^{T}\x_{i}-\U^{T}\x_{i}\right)^T\left(\U_{\theta}^{T}\x_{i}+\U^{T}\x_{i}\right)\\
    =&\left(\U^{T}\V\Rt^{T}\V^T\x_{i}-\U^{T}\V\V^T\x_{i}\right)^T\left(\U^{T}\V\Rt^{T}\V^T\x_{i}+\U^{T}\V\V^T\x_{i}\right)\\
    =&\x_{i}^T\V\left(\Rt^{T}-\I\right)^T\V^T\U\U^T\V\left(\Rt^{T}+\I\right)\V^T\x_{i}\\
    =&\x_{i}^T\V\left(\Rt^{T}-\I\right)^T\V^T\U\U^T\V\left(\Rt^{T}-\I\right)\V^T\x_{i}+\x_{i}^T\V\left(\Rt^{T}-\I\right)^T\V^T\U\U^T\V(2\I)\V^T\x_{i}\\
    =&\underbrace{\x_{i}^T\V\left(\Rt^{T}-\I\right)^T\V^T\U}_{\hat{x}_i^T}\underbrace{ \U^T\V\left(\Rt^{T}-\I\right)\V^T\x_{i}}_{\hat{x}_i}+2\underbrace{\x_{i}^T\V\left(\Rt^{T}-\I\right)^T\V^T\U}_{\hat{x}_i^T}\U^T\x_{i}\\
\end{align*}

\begin{equation*}
  \begin{split}
    {\hat{\x}_i}=&
    \U^T\V\left(\Rt^{T}-\I\right)\V^T\x_{i}\\
    =&{\U^T\V\left(\Rt^{T}-\I\right)\left<\v_1,\x_{i}\right>\e_1}\\
    =&\norm{\x_i}{\U^T\V\left(\Rt^{T}-\I\right)\e_1}\\
    =&\norm{\x_i}{\U^T\V\left(\cos(\theta)\e_1-\sin(\theta)\e_2-\e_1\right)}\\
    =&\norm{\x_{i}}\left({\U^T\v_1\left(\cos(\theta)-1\right)}-{\U^T\v_2\sin(\theta)}\right)\\
    =&{\U^T\x_i\left(\cos(\theta)-1\right)}-\norm{\x_{i}}{\U^T\v_2\sin(\theta)}
  \end{split}
\quad\quad
  \begin{split}
\hat{\x}_i^T\U^T\x_{i}
=&\x_i^T\U\left({\U^T\x_i\left(\cos(\theta)-1\right)}-\norm{\x_i}{\U^T\v_2\sin(\theta)}\right)\\
=&\norm{\U^T\x_{i}}^2\left(\cos(\theta)-1\right)-\norm{\x_i}{\frac{\x_i^T\U\U^T\x_j}{\norm{\U\U^T\x_j}}\sin(\theta)}\\
=&\norm{\U^T\x_{i}}^2\left(\cos(\theta)-1\right)
  \end{split}
\end{equation*}

\begin{align*}
\hat{\x}_i^T\hat{\x}_i
=&\norm{\x_{i}}^2\left({\U^T\v_1\left(\cos(\theta)-1\right)}-{\U^T\v_2\sin(\theta)}\right)^T\left({\U^T\v_1\left(\cos(\theta)-1\right)}-{\U^T\v_2\sin(\theta)}\right)\\
=&\norm{\x_{i}}^2\left({\v_1^T\U\U^T\v_1\left(\cos(\theta)-1\right)^2}+{\v_2^T\U\U^T\v_2\sin(\theta)^2}-2\sin(\theta)\v_2^T\U{\U^T\v_1\left(\cos(\theta)-1\right)}\right)\\
=&\norm{\x_{i}}^2\left({\v_1^T\U\U^T\v_1\left(\cos(\theta)-1\right)^2}+{\v_2^T\v_2\sin(\theta)^2}\right)\\
=&\x_i^T\U\U^T\x_i\left(\cos(\theta)-1\right)^2+\norm{\x_{i}}^2{\sin(\theta)^2}\\
=&\norm{\U^T\x_i}^2\cos(\theta)^2-2\norm{\U^T\x_i}^2\cos(\theta)+\norm{\U^T\x_i}^2+(\norm{\x_{i}}^2-\norm{\U^T\x_{i}}^2){\sin(\theta)^2}+\norm{\U^T\x_{i}}^2{\sin(\theta)^2}\\
=&\norm{\U^T\x_i}^2-2\norm{\U^T\x_i}^2\cos(\theta)+\norm{\U^T\x_i}^2+(\norm{\x_{i}}^2-\norm{\U^T\x_{i}}^2){\sin(\theta)^2}\\
=&(\norm{\x_{i}}^2-\norm{\U^T\x_{i}}^2){\sin(\theta)^2}+2\norm{\U^T\x_i}^2(1-\cos(\theta))\\
\end{align*}

So finally:
\begin{align*}
    f_i(\U)-f_i(\U_{\theta})=&\hat{\x}_i^T\hat{\x}_i+2\hat{\x}_i^T\U^T\x_{i}\\
    =&2\norm{\U^T\x_{i}}^2\left(\cos(\theta)-1\right)+(\norm{\x_{i}}^2-\norm{\U^T\x_{i}}^2){\sin(\theta)^2}+2\norm{\U^T\x_i}^2(1-\cos(\theta))\\
    =&(\norm{\x_{i}}^2-\norm{\U^T\x_{i}}^2){\sin(\theta)^2}
\end{align*}
.
\end{proof}

\end{figure*}

\begin{figure*}
\begin{lemma}
Let $\U\in\orth{d}{r}$, $\Rt=G(\theta,i,j)$ (a Given rotation over the axes i,j), $\U^{'}=\Rt\U$.
\begin{align*}
    \e_{i}^T\U^{'}\U^{'T}\e_{i}-\e_{i}^T\U\U^{T}\e_{i}=\sin\left(2\theta\right)\e_{i}^{T}\U\U^{T}\e_{j}+\sin^{2}\left(\theta\right)\left(\e_{j}^T\U\U^{T}\e_{j}-\e_{i}^T\U\U^{T}\e_{i}\right)
\end{align*}
\end{lemma}
\begin{proof}
Observe that:

\begin{align*}
    \e_{i}^T\U\U^{T}\e_{i}
    =&	\sum_{l=1}^{r}\U_{il}^{2}
\end{align*}
Since $\Rt\U\U^{T}\Rt^T=\U^{'}\U^{'T}$, we also have: 
\begin{align*}
    \e_{i}^T\U^{'}\U^{'T}\e_{i}
    =&\sum_{l=1}^{r}\left[\U^{T}\Rt^T\e_{i}\right]_{l}^{2}\\
    =&\sum_{l=1}^{r}\left[\U^{T}\left(\cos\left(\theta\right)\e_{i}+\sin\left(\theta\right)\e_{j}\right)\right]_{l}^{2}\\
    =&\sum_{l=1}^{r}\left(\cos\left(\theta\right)^2\U_{il}^{2}+2\cos\left(\theta\right)\U_{il}\sin\left(\theta\right)\U_{jl}+\sin\left(\theta\right)^{2}\U_{jl}^{2}\right)
\end{align*}

So:
\begin{align*}
    \e_{i}^T\U^{'}\U^{'T}\e_{i}-\e_{i}^T\U\U^{T}\e_{i}
    =&\sum_{l=1}^{r}\left(2\cos\left(\theta\right)\U_{il}\sin\left(\theta\right)\U_{jl}+\sin^{2}\left(\theta\right)\U_{jl}^{2}-\left(1-\cos^{2}\left(\theta\right)\right)\U_{il}^{2}\right)\\
    =& \sum_{l=1}^{r}\left(2\cos\left(\theta\right)\U_{il}\sin\left(\theta\right)\U_{jl}+\sin^{2}\left(\theta\right)\U_{jl}^{2}-\sin^{2}\left(\theta\right)\U_{il}^{2}\right)\\
    =& \sum_{l=1}^{r}\left(2\cos\left(\theta\right)\U_{il}\sin\left(\theta\right)\U_{jl}+\sin^{2}\left(\theta\right)\left(\U_{jl}^{2}-\U_{il}^{2}\right)\right)\\
    =& 2\cos\left(\theta\right)\sin\left(\theta\right)\sum_{l=1}^{r}\U_{il}\U_{jl}+\sin^{2}\left(\theta\right)\sum_{l=1}^{r}\left(\U_{jl}^{2}-\U_{il}^{2}\right)\\
    =& \sin\left(2\theta\right)\e_{i}^{T}\U\U^{T}\e_{j}+\sin^{2}\left(\theta\right)\left(\e_{j}^T\U\U^{T}\e_{j}-\e_{i}^T\U\U^{T}\e_{i}\right)
\end{align*}
\end{proof}

\end{figure*}

\end{document}